\documentclass{new_tlp}
\usepackage{mathptmx}
\usepackage{enumitem}
\DeclareMathAlphabet{\mathcal}{OMS}{cmsy}{m}{n}
\usepackage{dingbat}

\usepackage{amsmath,amssymb}
\usepackage{subfigure}
\usepackage{xspace}
\usepackage{enumitem}
\usepackage{IEEEtrantools}

\usepackage{url}
\usepackage[comments=on,changes=on]{krudces}
\usepackage{tikz}
\usetikzlibrary{automata,arrows,decorations.pathmorphing,backgrounds,fit,positioning}

\def\At{\text{\em At}}

\def\Not{\text{\tt not} \ }
\def\K{\mathbf{K}\, }
\def\M{\mathbf{M}\, }

\def\eligible{\mathit{eligible}}
\def\minority{\mathit{minority}}
\def\high{\mathit{high}}
\def\fair{\mathit{fair}}
\def\interview{\mathit{interview}}
\def\appointment{\mathit{appointment}}
\def\sM{\mathcal{M}}
\def\Atoms{\mathit{Atoms}}

\def\SM{\text{\rm SM}}
\newcommand\wv{\mathbb{W}}
\newcommand\wx{\mathbb{X}}
\def\cS{\mathcal{S}}
\newcommand\cset[1]{[#1]}
\newcommand\us{\mathbb{S}}
\newcommand\kdint[2]{({#2, #1})}

\def\Bodym{\mathit{Body}_{sub}}
\def\Bodyr{\mathit{Body}_{obj}}
\def\Head{\mathit{Head}}
\def\Body{\mathit{Body}}

\def\Bodyrp{\mathit{Body}^+_{ob}}
\def\Bodymp{\mathit{Body}^+_{sub}}
\def\bL{\K}

\title[Splitting Epistemic Logic Programs]{Splitting Epistemic Logic Programs
\thanks{
This is an extended version of a conference paper presented at the \emph{Fifteenth International Conference on Logic Programming and
Nonmonotonic Reasoning},  Philadelphia, PA, USA, June 3-7, 2019, where it received a best technical paper award~\protect\cite{cafafa19a}
An earlier version was also presented at the
\emph{Seventeenth International Workshop on Non-monotonic Reasoning}~\protect\cite{cafafa18}.
This work was partially supported by MINECO, Spain, grant TIC2017-84453-P, Xunta de Galicia, Spain (GPC ED431B 2019/03). 
The second author was funded by the Centre International de Math\'{e}matiques et d'Informatique de Toulouse (CIMI) through contract ANR-11-LABEX-0040-CIMI within the program ANR-11-IDEX-0002-02
and the Alexander von Humboldt Foundation.
}}
\author[Pedro Cabalar, Jorge Fandinno and Luis {Fari\~{n}as del Cerro}]{PEDRO CABALAR\\
University of Corunna, Spain\\
\email{cabalar@udc.es}
\and 
JORGE FANDINNO\\
Universit\"{a}t Potsdam, GERMANY\\
\email{fandinno@uni-potsdam.de}
\and
LUIS {FARI\~NAS~DEL~CERRO}\\
IRIT, University of Toulouse, CNRS, France\\
 \email{farinas@irit.fr}
}

\begin{document}
\maketitle

\begin{abstract}
Epistemic logic programs constitute an extension of the stable model semantics to deal with new constructs called \emph{subjective literals}.
Informally speaking, a subjective literal allows checking whether some objective literal is true in all or some stable models.
As it can be imagined, the associated semantics has proved to be non-trivial, since the truth of subjective literals may interfere with the set of stable models it is supposed to query.
As a consequence, no clear agreement has been reached and different semantic proposals have been made in the literature.
Unfortunately, comparison among these proposals has been limited to a study of their effect on individual examples, rather than identifying general properties to be checked.
In this paper, we propose an extension of the well-known splitting property for logic programs to the epistemic case.
We formally define when an arbitrary semantics satisfies the \emph{epistemic splitting property} and examine some of the consequences that can be derived from that, including its relation to conformant planning and to epistemic constraints.
Interestingly, we prove (through counterexamples) that most of the existing approaches fail to fulfill the epistemic splitting property, except the original semantics proposed by Gelfond in 1991 and a recent proposal by the authors, called \emph{Founded Autoepistemic Equilibrium Logic}.
\end{abstract}

\section{Introduction}

The language of \emph{epistemic specifications}, proposed by~\citeN{gelfond91a}, constituted an extension of disjunctive logic programming that introduced modal operators to quantify over the set of stable models~\cite{gellif88b} of a program.
These new constructs were later incorporated as an extension of the Answer Set Programming (ASP) paradigm in different implemented solvers (see~\citeNP{leckah18} for a recent survey).
The new constructs, \emph{subjective literals}, have the form~$\K l$ and~$\M l$ and allow respectively checking whether an objective literal~$l$ is true in every stable model (cautious consequence) or in some stable model (brave consequence).
In many cases, these subjective literals can be seen as simple queries, but what makes them really interesting is their use in rule bodies, which may obviously affect the set of stable models they are meant to quantify.
This feature makes them suitable for modelling introspection (reasoning about the knowledge and lack of knowledge that the system possesses rather than reasoning exclusively about the facts themselves) but, at the same time, easily involves cyclic specifications whose intuitive behaviour is not always easy to define.
For instance, the semantics of an epistemic logic program may yield alternative sets of stable models, each set being called a \emph{world view}.
Deciding the intuitive world views of a cyclic specification has motivated a wide debate in the literature.
In fact, in Gelfond's original \mbox{semantics (G91;~\citeNP{gelfond91a})} or in its extensions to arbitrary propositional formulas~\cite{wanzha05a,truszczynski11}, some cyclic examples manifested self-supportedness, so \citeN{gelfond11a} himself and, later on, other authors~\cite{kawabagezh15,faheir15a,sheeit17a,cafafa19b} proposed different variants trying to avoid unintended results.
Unfortunately, comparison among these variants was limited to studying their effect on a set of ``test'' examples,
leading to a lack of
confidence as any proposal is always subject to the appearance of new counterintuitive examples.
A next methodological step would consist in defining formal properties to be established
and that
would cover  complete families of examples and, hopefully, could help to reach an agreement on some language fragments.
For instance, one would expect that, at least, the existing approaches agreed on their interpretation of acyclic specifications.
Regretfully, as we will see, this is not the~case.

In this paper we propose a candidate property, namely \emph{epistemic splitting}, that not only defines an intuitive behaviour for stratified epistemic specifications but also goes further, extending the splitting theorem, well-known for standard logic programs~\cite{liftur94a}, to the case of epistemic logic programs.
In fact, the idea of splitting goes back to the result in~\cite{gelprz92a} for the case of Moore's autoepistemic logic~\cite{moore85}.
Informally speaking, we say that an epistemic logic program can be split if a part of the program (the \emph{top}) only refers to the atoms of the other part (the \emph{bottom}) through subjective literals.
A given semantics satisfies epistemic splitting if, given any split program, it is possible to get its world views by first obtaining the world views of the bottom and then using the subjective literals in the top as ``queries'' on the bottom part previously obtained.
If epistemic splitting holds, the semantics immediately satisfies other properties.
For instance, if the use of epistemic operators is stratified, the program has a unique world view at most.
Similarly, epistemic constraints (those only consisting of subjective literals) can be guaranteed to only rule out candidate world views.
However, we will see that, among the previously cited approaches, only the G91 semantics and the recently proposed \emph{Founded Autoepistemic Equilibrium Logic}~\cite{cafafa19b} satisfy epistemic splitting.
So, somehow, most of the recent attempts to fix the behaviour of cycles have neglected the attention on the effects produced on acyclic specifications.
In fact, a different property of epistemic splitting was already proved in~\cite{watson00} as a method to compute world views for the G91 semantics. 
However, that definition was based on a ``safety'' condition that needed to be checked for all possible world views and is specific for~G91 semantics, so it is harder to justify as a general property required for other approaches.

The rest of the paper is organised as follows.
First, we motivate the main idea through a \mbox{well-known} example.
After that, we recall basic definitions of (non-epistemic) ASP and splitting,
introduce the language of epistemic specifications and define the G91 semantics.
In the next section, we proceed to define the property of epistemic splitting and study some of its consequences.
Then, we formally prove that G91 satisfies this property while we provide counterexamples for the approaches that do not satisfy it.
Finally, before concluding the paper, we show how epistemic splitting can be applied to simplify the formalisation of conformant planning problems when representing them as epistemic logic programs.

\section{Motivation}\label{sec:motiv}

To illustrate the intuition behind our proposal, let us consider the following well-known standard example introduced in~\cite{gelfond91a}.

\begin{example}\label{ex:college}
A given college uses the following set of rules to decide whether a student $X$ is eligible for a scholarship:
\begin{eqnarray}
\eligible(X) & \leftarrow & \high(X)   \label{ex:college.1} \\
\eligible(X) & \leftarrow & \minority(X),\, \fair(X) \label{ex:college.2}\\
\sneg\eligible(X) & \leftarrow & \sneg\fair(X),\, \sneg\high(X) \label{ex:college.3}
\end{eqnarray}
Here, `$\sneg$' stands for strong negation and $\high(X)$ and $\fair(X)$ refer to the grades of student $X$.
We want to encode the additional college criterion
``\emph{The students whose eligibility is not determined by the college rules should be interviewed by the scholarship committee}'' 
as another rule in the program.\qed
\end{example}

\noindent The problem here is that, for deciding whether $eligible(X)$ ``\emph{can be determined},'' we need to check if it holds in all the answer sets of the program, that is, if it is one of the cautious consequences of the latter.
For instance, if the only available information for some student $mike$ is the disjunction
\begin{gather}
\fair(mike) \vee \high(mike) \label{ex:college.4}
\end{gather}
we get that program \mbox{$\set{\!\eqref{ex:college.1} \text{-} \eqref{ex:college.4}\!}$} has the following two stable models:
\begin{gather}
\set{ \high(mike), \eligible(mike) }
	\label{f:sm1.pre}
\\
\set{ \fair(mike) }
	\label{f:sm2.pre}
\end{gather}
so $\eligible(mike)$ cannot be determined and an interview should follow.
Of course, if we just want to query cautious and brave consequences of the program, we can do it inside ASP.
For instance, the addition of constraint:
\begin{eqnarray*}
\bot \leftarrow \eligible(mike)
\end{eqnarray*}
allows us to decide if $\eligible(mike)$ is a cautious consequence by just checking that the resulting program has no answer sets.
The difficulty comes from the need to \emph{derive} new information from a cautious consequence.
This is where subjective literals come into play.
Rule
\begin{eqnarray}
\begin{IEEEeqnarraybox}{rCl}
\interview(X) &\leftarrow  &\Not\!\K \eligible(X),\,
\Not\!\K \sneg \eligible(X) 
\end{IEEEeqnarraybox}
\label{ex:college.5}
\end{eqnarray}
allows us to prove that $\interview(X)$ holds whenever neither $\eligible(X)$ nor $\sneg eligible(X)$ are cautious consequences of \mbox{$\set{\eqref{ex:college.1} \text{-} \eqref{ex:college.4}}$}.
Recall that $\K l$ holds when the literal $l$ is true in all stable models of the program.
The novel feature here is that \eqref{ex:college.5} is also part of the program, and so, it affects the answer sets queried by $\K$ too, which would actually be:
\begin{eqnarray}
&\{\fair(mike),\interview(mike)\}& \label{f:sm1}\\
&\{\high(mike), \eligible(mike),\interview(mike)\} &\label{f:sm2}
\end{eqnarray}
So, there is a kind of cyclic reasoning: operators $\K$ and $\M$ are used to query a set of stable models that, in their turn, may depend on the application of that query.
In the general case, this kind of cyclic reasoning is solved by resorting to multiple world views, but in our particular example, however, this does not seem to be needed. 
One would expect that separating the queried part \mbox{$\set{\eqref{ex:college.1} \text{-} \eqref{ex:college.4}}$} and the rule that makes the query \eqref{ex:college.5} should be correct, since the first four rules do not depend on \eqref{ex:college.5} and the latter exclusively consults them without interacting with their results.
A similar line of reasoning could be applied if we added one more level such as, for instance, by including the rule:
\begin{eqnarray}
\appointment(X) \leftarrow  \K \interview(X) \label{ex:college.6}
\end{eqnarray}
The two answer sets of program~\mbox{$\set{\eqref{ex:college.1} -\eqref{ex:college.5}}$} contain $\interview(mike)$ and so $\appointment(mike)$ can be added to both answer sets incrementally.
This method of analysing a program by division into independent parts shows a strong resemblance to the \emph{splitting theorem}, well-known for standard ASP.
Splitting is applicable when the program can be divided into two parts, the \emph{bottom} and the \emph{top}, in such a way that the bottom never refers to head atoms in the top.
When this happens, we can first compute the stable models of the bottom and then, for each one, simplify the top accordingly, getting new stable models that complete the information.
We could think about different ways of extending this method for the case of epistemic logic programs, depending on how restrictive we want to be on the programs where it will be applicable.
However, we will choose a very conservative case, looking for a wider agreement on the proposed behaviour.
The condition we will impose is that our top program can only refer to atoms in the bottom through epistemic operators.
In this way, the top is seen as a set of rules that derive facts from epistemic queries on the bottom.
Thus, each world view $\wv$ of the bottom will be used to replace the subjective literals in the top by their truth value with respect to $\wv$.
For the sake of completeness, we recall next the basic definitions of ASP and splitting, to proceed with a formalisation of epistemic splitting afterwards.

\section{Background of ASP and Epistemic Specifications}
\label{sec:back}

Given a set of atoms $\At$, an \emph{objective literal} is either an atom or a truth constant\footnote{For a simpler description of program transformations, we allow truth constants where $\top$ denotes true and $\bot$ denotes false.}, that is \mbox{$a \in \At \cup \{\top,\bot\}$}, or an atom preceded by one or two occurrences of default negation,
\mbox{$\Not a$}.
An \emph{objective rule} $r$ is an implication of the form:
\begin{gather}
a_1 \vee \dots \vee a_n \leftarrow L_1, \dots, L_m
	\label{eq:rule}
\end{gather}
with $n\geq 0$ and $m\geq 0$, where each $a_i \in \At$ is an atom and each $L_j$ an objective literal.
The left hand disjunction of \eqref{eq:rule} is called the rule \emph{head} and abbreviated as $\Head(r)$.
When $n=0$, it corresponds to $\bot$ and $r$ is called a \emph{constraint}.
The right hand side of \eqref{eq:rule} is called the rule \emph{body} and abbreviated as $\Body(r)$.
When $m=0$, the body corresponds to $\top$ and $r$ is called a \emph{fact} (in this case, the body and the arrow symbol are usually omitted).
An \emph{objective program} $\Pi$ is a (possibly infinite) set of objective rules.
We write $\Atoms(F)$ to represent the set of atoms occurring in any syntactic construct $F$ (a literal, head, body, rule or program).
By abuse of notation, we will sometimes respectively write $\Head(r)$ and $\Body(r)$ instead of
$\Atoms(\Head(r))$ and $\Atoms(\Body(r))$ when it is clear by the context.
A propositional interpretation $I$ is a set of atoms.
We assume that strong negation `$\sneg a$' is just another atom in $\At$ and that the constraint
$\bot \leftarrow a, \sneg a$
is implicitly included in the program.
We allow the use of variables, but understood as abbreviations of their possible ground instances.
Given any syntactic construct $F$, we write $I \models F$ to stand for ``$I$ satisfies $F$'' in classical propositional logic, where the commas correspond to conjunctions, `$\Not \!\!$' corresponds (under this interpretation) to classical negation and `$\leftarrow$' is just a reversed material implication.
An interpretation $I$ is a \emph{(classical) model} of an objective program $\Pi$ if it satisfies all its rules.
The \emph{reduct} of a program $\Pi$ with respect to some propositional interpretation~$I$,
in symbols $\Pi^I$, is obtained by replacing in $\Pi$ every negative literal $\Not a$ by $\top$ if $I \models \Not a$ or by $\bot$ otherwise.
A propositional interpretation~$I$ is a \emph{stable model} of a program $\Pi$ iff it is a $\subseteq$-minimal model of $\Pi^I$.
By $\SM[\Pi]$, we denote the set of all stable models of $\Pi$.
The following is a well-known property in ASP.

\begin{property}[Supraclassicality]\label{prop:supraclass}
Any stable model of any objective program $\Pi$ is also a classical model of $\Pi$. 
\end{property}

We recall next the splitting theorem for ASP, beginning with the following definition.
\begin{definition}[Splitting set]\label{def:splitset}
A set of atoms $U \subseteq \At$ is a \emph{splitting set} of a program $\Pi$ if, for each rule $r \in \Pi$, one of the following conditions hold
\begin{enumerate}[ label=(\roman*), leftmargin=15pt]
\item $\Atoms(r) \subseteq U$, or
\item $\Head(r) \cap U = \emptyset$.
\end{enumerate}
When this happens, we identify two disjoint subprograms, the \emph{bottom} and the \emph{top}, respectively defined as
\begin{IEEEeqnarray*}{lCl ?C? lCl +x*}
b_U(\Pi) & \eqdef & \setm{ r \in \Pi}{ \Atoms(r) \subseteq U }
&\hspace{1.5cm}&
t_U(\Pi) & \eqdef & \Pi \setminus b_U(\Pi)  &\qed
\end{IEEEeqnarray*}
\end{definition}
As an example, consider program $\Pi_1$:
\begin{IEEEeqnarray}{cCl ?C? rCl}
a & \leftarrow & \Not b \label{f:ex1.1}
&\hspace{2cm}&
b & \leftarrow & \Not a \label{f:ex1.2}\\
c \vee d & \leftarrow & \Not a \label{f:ex1.3}
&\hspace{2cm}&
d & \leftarrow & a, \Not b \label{f:ex1.4}
\end{IEEEeqnarray}
It is easy to see that $U=\{a,b\}$ is a splitting set that divides the program into two parts: the bottom $b_U(\Pi_1)
$ containing the rules~\eqref{f:ex1.1} and the top $t_U(\Pi_1)$ containing the rules~\eqref{f:ex1.3}.
The keypoint of splitting is computing stable models of $b_U(\Pi)$ alone and using each one, $I$, to simplify $t_U(\Pi)$ accordingly.
Given a splitting set $U$ for $\Pi$ and an interpretation $I \subseteq U$, we define the program $e_U(\Pi,I)$ as a transformation of the top program, $t_U(\Pi)$, where we replace each atom $a \in U$ from the splitting set by: $\top$ if $a \in I$ or $\bot$ otherwise.
A pair $\tuple{I_b,I_t}$ is said to be a \emph{solution} of $\Pi$ with respect to $U$
iff $I_b$ is a stable model of $b_U(\Pi)$
and $I_t$ is a stable model of $e_U(\Pi,I_b)$.
For instance, for $\Pi_1$, the bottom has two answer sets $\{a\}$ and $\{b\}$, so we get the respective simplifications $e_U(\Pi_1,\{a\})$:
\begin{eqnarray*}
c \vee d \leftarrow \Not \top\hspace{30pt}
d \leftarrow \top, \Not \bot
\end{eqnarray*}
and $e_U(\Pi_1,\{b\})$:
\begin{eqnarray*}
c \vee d \leftarrow \Not \bot\hspace{30pt}
d \leftarrow \bot, \Not \top
\end{eqnarray*}
The former has stable model $\{d\}$ so $\tuple{\{a\},\{d\}}$ is one solution.
The latter has stable models $\{c\}$ and $\{d\}$ that yield other two solutions $\tuple{\{b\},\{c\}}$ and $\tuple{\{b\},\{d\}}$.

\begin{theorem}[From~\protect\citeNP{liftur94a}]\label{thm:nm.splitting}
Let $U$ be a splitting set of program~$\Pi$.
A propositional interpretation $I \subseteq \text{\em \At}$ is a stable model of $\Pi$ iff 
there is a solution $\tuple{I_b,I_t}$ of $\Pi$ w.r.t~$U$ such that $I = I_b \cup I_t$.\qed
\end{theorem}
Given the three solutions we obtained before, the splitting theorem guarantees that $\{a\} \cup \{d\}$,  $\{b\} \cup \{c\}$ and $\{b\}\cup\{d\}$ are the three stable models of our example program $\Pi_1$.

One interesting observation is that any constraint $r$ with \mbox{$\Atoms(r) \subseteq U$}
\textbf{}is now included in the bottom $b_U(\Pi)$ but also satisfies condition (ii) in Def.~\ref{def:splitset} (it has no head atoms at all) and could be moved to the top $t_U(\Pi)$ instead.
Having this in mind, let us provide now a relaxed definition of bottom and top programs in the following way.
We say that the pair $\tuple{\hat{b}_U(\Pi),\hat{t}_U(\Pi)}$ is an \emph{arbitrary splitting} of program $\Pi$ with respect to splitting set $U$ if: $\hat{b}_U(\Pi) \cap \hat{t}_U(\Pi) = \emptyset$, $\hat{b}_U(\Pi) \cup \hat{t}_U(\Pi) = \Pi$, all rules in $\hat{b}_U(\Pi)$ satisfy condition~(i) in Definition~\ref{def:splitset} and all rules in $\hat{t}_U(\Pi)$ satisfiy condition~(ii) in Definition~\ref{def:splitset}.
With this definition, constraints on atoms in $U$ can be arbitrarily placed in $\hat{b}_U(\Pi)$ or in $\hat{t}_U(\Pi)$.

\begin{corollary}\label{cor:nm.splitting}
Theorem~\ref{thm:nm.splitting} still holds if we define $b_U(\Pi)$ and $t_U(\Pi)$ to be any arbitrary splitting $\tuple{\hat{b}_U(\Pi),\hat{t}_U(\Pi)}$ of program $\Pi$ w.r.t. splitting set $U$.
\qed
\end{corollary}

Note also that, given an interpretation~$I$,
the solution~$\tuple{I_b,I_t}$ is unique and satisfies ${I_b = I \cap U}$ and ${I_t = I \setminus U}$.
As a result,
we obtain the following rewriting of Theorem~\ref{thm:nm.splitting}:

\begin{corollary}\label{cor:nm.splitting2}
Let $U$ be a splitting set of program~$\Pi$.
A propositional interpretation $I \subseteq \text{\em \At}$ is a stable model of $\Pi$ iff 
$\tuple{I\cap U,I \setminus U}$ is a solution of $\Pi$ w.r.t~$U$.\qed
\end{corollary}


We extend now the syntax of ASP to the language of epistemic specifications.
Any expression of the form $\K l$, $\M l$, $\Not \K l$ or $\Not \M l$,
with $l$ an objective literal,
is called a \emph{subjective literal}.
We keep the same syntax for rules as in \eqref{eq:rule} except that body literals $L_j$ can also be subjective literals now.
Given rule $r$, we define the sets $\Bodyr(r)$ and $\Bodym(r)$ respectively containing the objective and the subjective literals in $\Body(r)$.
We say that a rule is a \emph{subjective constraint} if it is a constraint (that is, ${\Head(r)=\bot}$) and its body exclusively consists of subjective literals (that is, $\Body(r)=\Bodym(r)$\,).

We can define the concept of \emph{model} of a program, in a similar way as we did for classical models in objective ASP.
A \emph{modal interpretation} $\sM = \tuple{\wv,I}$ is pair where $I$ is a propositional interpretation and $\wv \subseteq 2^{\At}$ is a non-empty set of propositional interpretations.
A modal interpretation $\sM = \tuple{\wv,I}$ \emph{satisfies} a literal if one of the following conditions hold
\begin{enumerate}
\item $\tuple{\wv,I} \models \top$,
\item $\tuple{\wv,I} \not\models \bot$,
\item $\tuple{\wv,I} \models a$ if $a \in I$, for any atom $a \in \At$,
\item $\tuple{\wv,I} \models \K l$ if $\tuple{\wv,I'} \models l$ for all $I' \in \wv$,
\item $\tuple{\wv,I} \models \M l$ if $\tuple{\wv,I'} \models l$ for some $I' \in \wv$, and
\item $\tuple{\wv,I} \models \Not L$ if $\tuple{\wv,I} \not\models L$.
\end{enumerate}
For a subjective literal $L$, $\tuple{\wv,I} \models L$ does not depend on $I$, so we sometimes write
just $\wv \models L$ instead of $\tuple{\wv,I} \models L$.
For a rule $r$ of the form~\eqref{eq:rule}, we write $\tuple{\wv,I} \models r$ iff either
$\tuple{\wv,I} \models a_i$ for some $1 \leq i \leq n$ or $\tuple{\wv,I} \not\models L_j$ for some $1 \leq j \leq m$.
We say that $\tuple{\wv,I}$ is a \emph{model} of a program~$\Pi$, written $\tuple{\wv,I}\models \Pi$, if it satisfies all its rules.
Among the possible models of an epistemic logic program, all semantic approaches agree on selecting some preferred models called \emph{world views}, each one being characterized by the $\wv$ component.
These world views satisfy a similar property to that of supraclassicality (Property~\ref{prop:supraclass}) in non-epistemic ASP.
In this case, however, rather than talking about classical models, we resort to the semantics of modal logic S5, so all world views of a program are also S5 models of the program.
Formally,
we say that~$\wv$ is a S5-model of a program~$\Pi$, in symbols~$\wv \models \Pi$,
when~$\tuple{W,I} \models \Pi$ for every $I \in \wv$.
Then, this property can be formally stated as follows:
\begin{property}[Supra-S5]\label{property:supraS5}
A semantics $\cS$ satisfies \emph{supra-S5} when every $\cS$-world view $\wv$ of an epistemic program $\Pi$ is also a S5-model of~$\Pi$. \qed
\end{property}

Notice that this definition is semantics-dependent in the sense that each alternative semantics~$\cS$ for epistemic specifications may define different $\cS$-world views,
although, to the best of our knowledge, all existing semantics satisfy supra-S5.
Another property that is shared by all semantics is that, when $\Pi$ is an objective ASP program (it has no modal epistemic operators) then it has a unique world view containing all the stable models of $\Pi$.
We will formalize this property in the following way.
\begin{property}[Supra-ASP]\label{property:supraASP}
A semantics $\cS$ satisfies \emph{supra-ASP} if for any objective program $\Pi$ either $\Pi$ has a unique $\cS$-world view $\wv=\SM[\Pi] \neq \emptyset$ or both $\SM[\Pi]=\emptyset$ and $\Pi$ has no world view at all. \qed
\end{property}

Originally, some semantics like \cite{gelfond91a} or \cite{truszczynski11}, allowed empty world views $W=\emptyset$ when the program has no stable models, rather than leaving the program without world views.
Since this feature is not really essential, we exclusively refer to non-empty world views in this paper. 

We define next a useful transformation extending the idea of reduct to epistemic specifications, but generalised for a given signature.

\begin{definition}[Subjective reduct]
The \emph{subjective reduct} of a program $\Pi$ with respect to a set of propositional interpretations~$\wv$ and a signature $U \subseteq \At$,
also written $\Pi^\wv_U$, is obtained by replacing each subjective literal $L$ with $\Atoms(L) \subseteq U$ by; $\top$ if $\wv \models L$ or by $\bot$ otherwise.
When $U=\text{\rm \At}$ we just write $\Pi^\wv$.\qed
\end{definition}
\noindent We use the same notation $\Pi^\wv$ as for the standard reduct, but ambiguity is removed by the type of $\wv$ (a set of interpretations now).
This subjective reduct can be used to define the G91 semantics in the following way.
\begin{definition}[G91-world view]
A non-empty set of interpretations $\wv$ is a \emph{G91-world view} of an epistemic program $\Pi$ if $\wv=\SM[\Pi^\wv]$.\qed
\end{definition}

The definitions of the other semantics studied in the paper are delayed until Section~\ref{sec:related} where we prove one by one whether it satisfies epistemic splitting or not.

\section{Epistemic splitting}

We proceed now to introduce our definition of the epistemic splitting property.
To do so, we begin by extending the idea of splitting set from~\cite{liftur94a}.
\begin{definition}[Epistemic splitting set]\label{def:splitting}
A set of atoms \mbox{$U \subseteq \At$} is said to be an \emph{epistemic splitting set} of a program $\Pi$ if for any rule $r$ in $\Pi$ one of the following conditions hold
\begin{enumerate}[ label=(\roman*), leftmargin=18pt]
\item $\Atoms(r) \subseteq U$,
    \label{item:1:def:splitting}
\item 
$(\Bodyr(r) \cup \Head(r)) \cap U = \emptyset$.
\label{item:3:def:splitting}%
\end{enumerate}
We define a \emph{splitting} of 
$\Pi$ as a pair $\tuple{B_U(\Pi),T_U(\Pi)}$ satisfying $B_U(\Pi) \cap T_U(\Pi) = \emptyset$ and $B_U(\Pi) \cup T_U(\Pi) = \Pi$,
and also that all rules in $B_U(\Pi)$ satisfy (i) and all rules in $T_U(\Pi)$ satisfy (ii).
\qed
\end{definition}

\noindent 
With respect to the original definition of splitting set, we have replaced the condition  for the top program, $\Head(r) \cap U = \emptyset$, by the new condition~(ii) which
essentially means that the top program may only refer to atoms~$U$ in the bottom through epistemic operators.
Note that this introduces a new kind of ``dependence,'' so that, as happens with head atoms, objective literals in the body also depend on atoms in subjective literals.
For instance, if $U=\{p,q\}$, the program
$\newprogram\label{prg:dependence} = \set{ p \vee q  \ , \ 
s \leftarrow p, \K q}$
would not be splittable due to the second rule, since $s \not\in U$ and we would also need the objective literal $p \not\in U$.
The reason for this restriction is to avoid imposing (to a potential semantics) a fixed way of evaluating $p$ with respect to the world view $[\{p\},\{q\}]$ for the bottom.

Another observation is that we have kept the definition of $B_U(\Pi)$ and $T_U(\Pi)$ non-deterministic in the sense of the arbitrary splitting from the previous section: some rules can be arbitrarily included in one set or the other.
These rules correspond to subjective constraints on atoms in~$U$, since these are the only cases that may satisfy conditions~(i) and~(ii) simultaneously.


If we retake our example program
\mbox{$\newprogram\label{prg:2} = \set{\eqref{ex:college.1}-\eqref{ex:college.5}}$},
we can see that the set
$U$ consisting of atoms $high(mike), fair(mike),$ $eligible(mike), minority(mike)$ and their corresponding strong negations is an epistemic splitting set that divides the program into a bottom
\mbox{$B_U(\program\ref{prg:2})=\set{\eqref{ex:college.1}-\eqref{ex:college.4}}$}
and top part
\mbox{$T_U(\program\ref{prg:2})=\set{\eqref{ex:college.5}}$}.
As in objective splitting, the idea is computing first the world views of the bottom program $B_U(\Pi)$ and for each one simplifying the corresponding subjective literals in the top program.
Given an epistemic splitting set~$U$ for a program~$\Pi$ and a set of interpretations~$\wv$, we define $E_U(\Pi,\wv) \eqdef T_U(\Pi)^\wv_U$, that is, we make the subjective reduct of the top with respect to $\wv$ and signature $U$.

\begin{definition}
A pair $\tuple{\wv_b,\wv_t}$ is said to be an $\cS$-\emph{solution} of $\Pi$ with respect to an epistemic splitting set $U$ if $\wv_b$ is a $\cS$-world view of $B_U(\Pi)$ and $\wv_t$ is a $\cS$-world view of $E_U(\Pi,\wv_b)$.\qed
\end{definition}

As with the above properties, this definition is semantics-dependent in the sense that each alternative semantics~$\cS$ for epistemic specifications may define its own $\cS$-solutions for a given $U$ and~$\Pi$, since it may define the selected $\cS$-world views for a program in a different way.
Back to our example, notice that $B_U(\Pi_2)$ is an objective program without epistemic operators.
Thus, any semantics satisfying supra-ASP will provide $\wv_b=[\{fair(mike)\}, \{high(mike), eligible(mike)\}]$ as the unique world view for the bottom.
The corresponding simplification of the top would be $E_U(\program\ref{prg:2},\wv_b)$ containing (after grounding) the single rule
$
\interview(mike) \leftarrow \Not \bot, \Not \bot
$.
Again, this program is objective and its unique world view would be $\wv_t=[\{interview(mike)\}]$.
Now, in the general case, to reconstruct the world views for the global program we define the operation:
$$\wv_b \sqcup \wv_t \ \  = \ \ \setm{I_b \cup I_t}{ I_b \in \wv_b  \text{ and } I_t \in \wv_t  }$$
(remember that both the bottom and the top may produce multiple world views, depending on the program and the semantics we choose).
In our example, $\wv_b \sqcup \wv_t$ would exactly contain the two stable models $\eqref{f:sm1},\eqref{f:sm2}$ we saw in the introduction.

\begin{property}[Epistemic splitting]
A semantics $\cS$ satisfies \emph{epistemic splitting} if for any epistemic splitting set $U$ of any program $\Pi$: $\wv$ is an $\cS$-world view of $\Pi$ iff there is an $\cS$-solution $\tuple{\wv_b,\wv_t}$ of~$\Pi$ with respect to $U$ such that
$\wv=\wv_b \sqcup \wv_t$. \qed
\end{property}

Back to the example, it can be easily seen that the world view we obtained in two steps is indeed the unique world view of the whole program, under any semantics satisfying epistemic splitting.
Uniqueness of world view was obtained in this case because both the bottom program $B_U(\Pi_2)$ and the top, after simplification, $E_U(\Pi_2,\wv_b)$ were objective programs and we assumed supra-ASP.
In fact, as we see next, we can still get a unique world view (at most) when there are no cyclic dependences among subjective literals.
This mimics the well-known result for \emph{stratified negation} in logic programming.
Let us define a modal dependence relation among atoms in a program $\Pi$ so that $dep(a,b)$ is true iff there is a rule $r \in \Pi$ such that $a \in (\Head(r) \cup \Bodyr(r))$ and $b \in \Bodym(r)$.

\begin{definition}
We say that an epistemic program $\Pi$ is \emph{epistemically stratified} if we can assign an integer mapping $\lambda: \At \to \mathbb{N}$ to each atom such that
\begin{enumerate}
    \item \mbox{$\lambda(a) = \lambda(b)$} for any rule $r \in \Pi$ and atoms $a,b \in (\Atoms(r)\setminus \Bodym(r))$,
    \item \mbox{$\lambda(a)>\lambda(b)$} for any pair of atoms $a,b$ satisfying $dep(a,b)$.\qed
\end{enumerate}
\end{definition}

Take, for instance, the program $\newprogram\label{prg:3}=\{\eqref{ex:college.1}-\eqref{ex:college.5},\eqref{ex:college.6}\}$.
We can assign atoms $\high(mike)$, $\fair(mike)$, $\minority(mike)$ and $\eligible(mike)$ layer 0.
Then $\interview(mike)$ could be assigned layer 1 and, finally, $\appointment(mike)$ can be located at layer 2.
So, $\program\ref{prg:3}$ is epistemically stratified.

\begin{theorem}
Let $\cS$ be  any semantics satisfying supra-ASP and epistemic splitting
and let $\Pi$ be a finite, epistemically stratified program.
Then, if $\Pi$ has some $\cS$-world view, \textbf{}this is unique.\qed
\end{theorem}

\begin{proof}
Let~$\lambda$ be an integer mapping.
We assume without loss of generality that for every $1 \leq i\leq j$,
if there is some atom~$a \in \at$ such that $\lambda(j) = a$,
then there is also some atom~$b \in \at$ such that $\lambda(i) = b$.
Let $U_i' \ \eqdef \ \setm{ a \in \at }{\lambda(a) = i }$
and let
$$\Pi_i' \ \eqdef \ \setm{ r \in \Pi }{ (\Atoms(r)\setminus \Bodym(r)) \subseteq U_i'  }$$
and let us also define the sets
$U_i \eqdef U_1' \cup \dotsc \cup U_i'$
and
$\Pi_i \eqdef \Pi_1' \cup \dotsc \cup \Pi_i'$.
Then, $\Pi_1', \dotsc , \Pi_n'$ is a partition of $\Pi$
and, for $1 < i \leq n$, we can see that $U_i$ is a splitting set of the program $\Pi_i$
with $B_{U_i}(\Pi_i) = \Pi_{i-1}$ and $T_{U_i}(\Pi_i) = \Pi_i'$.
Furthermore, since $\cS$ satisfies epistemic splitting, it follows that
$\wv_i$ is a \mbox{$\cS$-world} view of $\Pi_i$ iff there is a \mbox{$\cS$-solution} $\tuple{\wv_{i-1},\wv_i'}$ of $\Pi_i$ with respect to $U_i$ such that \mbox{$\wv_i = \wv_{i-1} \sqcup \wv_i'$}.
Note that, by definition of solution,
we get that
$\wv_{i-1}$ is a \mbox{$\cS$-world} view of~$B_{U_i}(\Pi_i)=\Pi_{i-1}$
and
$\wv_{i}'$ is a \mbox{$\cS$-world} view of~$E_{U_i}(\Pi_i,\wv_{i-1})$.
Then,
by induction hypothesis,
we get that
$\wv_{i-1} = \wv_1' \sqcup \dotsc \sqcup \wv_{i-1}'$
and, thus,
we get
$\wv_i = \wv_1' \sqcup \dotsc \sqcup \wv_i'$.
As a result,
it follows
\mbox{$\wv = \wv_n = \wv_1' \sqcup \dotsc \sqcup \wv_n'$}.
Furthermore,
it is easy to see that $\Pi_1 = \Pi_1'$ must be an objective program (possibly empty)
and, since $\cS$ satisfies \mbox{supra-ASP}, we get that it has a unique $\cS$-world view~$\wv_1$ at most.
Moreover, for $1 < i \leq n$, we get that $E_{U_i}(\Pi_i',\wv_{i-1})$ is also an objective program and, thus, 
it has a unique $\cS$-world view $\wv_i'$ at most.
Then, it follows by induction that $\wv$ is unique, if it exists.
\end{proof}

The proof of the theorem just relies on multiple applications of splitting to each layer and the fact that each simplification $E_U(\Pi,\wv_b)$ will be an objective program.
This is very easy to see in the extended example $\program\ref{prg:3}$.
We can split the program using as $U$ all atoms but $\appointment(mike)$ to get a bottom $\program\ref{prg:2}$ and a top $\set{\eqref{ex:college.6}}$.
Program $\program\ref{prg:2}$ can be split in its turn as we saw before, producing the unique world view 
$\cset{\eqref{f:sm1},\eqref{f:sm2}}$.
Then $E_U(\program\ref{prg:3},[ \eqref{f:sm1},\eqref{f:sm2} ])$ contains the single rule 
$$\appointment(mike) \leftarrow \top$$
This is also an objective program whose unique world view is $[\{\appointment(mike)\}]$ and, finally, the combination of these two world views yields again a unique world view 
$$[\ \eqref{f:sm1} \cup \{\appointment(mike)\}\ , \ \eqref{f:sm2} \cup \{\appointment(mike)\} \ ]$$

Supra-ASP and epistemic splitting not only guarantee the existence of (at most) one world view for epistemic stratified programs but, in fact, they also force all semantics to coincide on this class of programs~\cite[Theorem~7]{fandinno19} --
that is, either none has a world view or all of them yield the same, unique world view.
Another consequence of epistemic splitting is that subjective constraints will have a monotonic behaviour.
Note first that, for a subjective constraint~$r$, we can abbreviate $\tuple{\wv,I} \models r$ as $\wv \models r$ since the $I$ component is irrelevant.
Additionally, $\wv \models r$ means that $\Body(r)=\Bodym(r)$ is falsified, since $\Head(r)=\bot$.

\begin{property}[subjective constraint monotonicity]
A semantics satisfies \emph{subjective constraint monotonicity} if, for any epistemic program $\Pi$ and any subjective constraint $r$, $\wv$ is a world view of $\Pi \cup \{r\}$ iff both $\wv$ is a world view of $\Pi$ and $\wv \models r$.\qed
\end{property}

\begin{Theorem}{\label{thm:splitting->constraint.monotonicity}}
Epistemic splitting implies subjective constraint monotonicity. \qed
\end{Theorem}

\begin{proof}
Suppose we use a semantics satisfying epistemic splitting.
For any program $\Pi$ and any epistemic constraint $r$, we can always take the whole set of atoms $U=\Atoms(\Pi \cup \{r\})$ as an epistemic splitting set for $\Pi'=\Pi \cup \{r\}$ and take $B_U(\Pi')=\Pi$ and $T_U(\Pi')=\{r\}$.
For any world view $\wv$ of $B_U(\Pi')$ two things may happen.
A first possibility is $\wv \models r$, and so the body of $r$ has some false subjective literal in $\wv$, so $E_U(\Pi',\wv)$ would be equivalent to $\bot \leftarrow \bot$.
Then, the unique world view for the top would be $\wv_t=[\emptyset ]$ and $\wv \sqcup \wv_t=\wv$.
A second case is $\wv \not\models r$, so all literals in the body are satisfied and $E_U(\Pi',\wv)$ would be equivalent to $\bot \leftarrow \top$ which has no world views.
To sum up, we get exactly those world views $\wv$ of $\Pi$ that satisfy $r$.
\end{proof}

\section{Epistemic splitting in existing semantics}
\label{sec:related}

In this section, we study the property of epistemic splitting for the approaches mentioned in the introduction and start by proving that the G91 semantics does satisfy this property.

\subsection{Epistemic splitting in Gelfond's original semantics}
\label{sec:espliting.g91}

We begin introducing some notation.
Let $U \subseteq \At$ be any set of atoms and $\overline{U}=\At \setminus U$.
We define the projection of a set $\wv$ of propositional interpretations to atoms in $U$ as $\restr{\wv}{U} \eqdef \{I \cap U \mid I \in \wv\}$.
Using these abbreviations, we can observe that: 

\begin{observation}\label{prop:splitting.modal.atom}
Let $\wv$ be a set of propositional interpretations
and
$U\subseteq \At$ be a set of atoms.
Then, for any subjective literal~$L$ with $\Atoms(L)=\{a\}$:
\begin{enumerate}[ label=\roman*)]
\item if $a \in U$, then $\wv \models L$ iff $\restr{\wv}{U} \models L$,
\item if $a \notin U$, then $\wv \models L$ iff $\restr{\wv}{\overline{U}} \models L$.
\end{enumerate}
\end{observation}

\begin{lemma}\label{lem:splitting.reduct}
Let $\Pi$ be a program that accepts an epistemic splitting set \mbox{$U \subseteq \text{\rm \At}$} and let
$\wv$ be a set of propositional interpretations.
Let  $\wv_b = \restr{\wv}{U}$ and $\wv_t = \restr{\wv}{\overline{U}}$.
Then, we get
\begin{enumerate}[ label=\roman*), leftmargin=15pt]
\item $B_U(\Pi)^{\wv} = B_U(\Pi)^{\wv_b}$,
\item $T_U(\Pi)^\wv = E_U(\Pi,\wv_b)^{\wv_t}$, and
\item $\Pi^\wv =  B_U(\Pi)^{\wv_b} \cup E_U(\Pi,\wv_b)^{\wv_t}$.
\end{enumerate}
\end{lemma}

\begin{proof}
First, since every rule \mbox{$r \in B_U(\Pi)$}
satisfies \mbox{$\Atoms(\Bodym(r)) \subseteq U$},
from Observation~\ref{prop:splitting.modal.atom},
it follows that
$B_U(\Pi)^\wv = B_U(\Pi)^{\wv_b}$.
Furthemore,
for any program $\Gamma$, it is easy to check that
$\Gamma^\wv = (\Gamma^{\wv_b}_U)^{\wv_t}$,
that is, applying the reduct with respect to $\wv$ is the same as applying it with respect to its projection in $U$ and, afterwards, to the remaining part.
Thus, we get
$T_U(\Pi)^\wv = (T_U(\Pi)^{\wv_b}_U)^{\wv_t} = E_U(\Pi,\wv_b)^{\wv_t}$.
Finally, we have that
$\Pi^\wv = (B_U(\Pi) \cup T_U(\Pi))^\wv = B_U(\Pi)^\wv \cup T_U(\Pi)^\wv$
and, thus, the result holds.
\end{proof}

We are now ready to prove the main theorem.

\begin{mtheorem}
Semantics G91 satisfies epistemic splitting.\qed
\end{mtheorem}

\begin{proof}
Let $\wv$ be some set of propositional interpretations and let
$\wv_b = \restr{\wv}{U}$ and $\wv_t = \restr{\wv}{\overline{U}}$.
By definition, $\wv$ is a world view of $\Pi$
if and only if $\wv = \SM[\Pi^\wv]$.
Furthermore, since $U$ is a modal splitting set of $\Pi$,
it is easy to check that $U$ is also a regular splitting set of the regular program~$\Pi^\wv$.
Hence, from Corollary~\ref{cor:nm.splitting},
we get that $\wv$ is a world view of $\Pi$ iff
$$
\wv \ \ = \ \ \SM[\Pi^\wv] \ \ = \ \ \setbm{ I_b \cup I_t }{ I_b \in \SM[\hat{b}_U(\Pi^\wv)] \text{ and } I_t \in \SM[\hat{e}_U(\Pi^\wv\!,I_b)] }$$
for some arbitrary splitting $\tuple{\hat{b}_U(\Pi^\wv),\hat{t}_U(\Pi^\wv)}$.
Note that all rules belonging to $B_U(\Pi)$ have all atoms from~$U$.
Hence,
we take $\hat{b}_U(\Pi^\wv) \eqdef B_U(\Pi)^\wv = B_U(\Pi)^{\wv_b}$ (Proposition~\ref{lem:splitting.reduct}).
Similarly, we also take
$\hat{t}_U(\Pi^\wv) \eqdef T_U(\Pi)^\wv = E_U(\Pi,\wv_b)^{\wv_t}$.
Then, we get
$$\hat{e}_U(\Pi^\wv\!,I_b) \ = \  \hat{e}_U(\hat{t}_U(\Pi^\wv),I_b) \ = \ \hat{e}_U(E_U(\Pi,\wv_b)^{\wv_t}\!,I_b)$$
Notice also that no atom occurring in $E_U(\Pi,\wv_b)^{\wv_t}$ belongs to $U$,
which implies that
$$\hat{e}_U(E_U(\Pi,\wv_b)^{\wv_t}\!,I_b) = E_U(\Pi,\wv_b)^{\wv_t}$$
Replacing above, we have that
$\wv$ is a world view of $\Pi$ iff $\wv$ is equal to
$$\{ I_b \cup I_t \mid I_b \in \SM[B_U(\Pi)^{\wv_b}], I_t \in \SM[E_U(\Pi,\wv_b)^{\wv_t}] \}$$
iff
$$\wv = \{ I_b \cup I_t \mid I_b \in \wv_b' \text{ and } I_t \in \wv_t' \}$$
with $\wv_b' = \SM[B_U(\Pi)^{\wv_b}]$
and
$\wv_t '= \SM[E_U(\Pi,\wv_b)^{\wv_t}]$
\\iff
$\wv = \wv_b' \sqcup \wv_t'$.
Hence, it only remains to be shown that both
$\wv_b = \wv_b'$
and
$\wv_t = \wv_t'$ hold.
Note that $I \in \wv_b = \restr{\wv}{U}$ iff $I = I' \cap U$ for some $I' \in \wv$
iff
$I = (I_b \cup I_t) \cap U$ for some $I_b \in \wv_b'$ and $I_t \in \wv_t'$
iff
$I = (I_b \cap U) \cup (I_t \cap U)$ for some $I_b \in \wv_b'$ and $I_t \in \wv_t'$
iff
$I = I_b$ for some $I_b \in \wv_b'$.
The fact $\wv_t = \wv_t'$ follows in an analogous way.
\end{proof}

Finally, recall that \cite{wanzha05a} and~\cite{truszczynski11} extended the G91 semantics to cover arbitrary formulas, but on the syntax of epistemic logic programs these three semantics agree. Consequently, this proof applies directly to these semantics.

\begin{corollary}
The semantics by~\citeN{wanzha05a} and~\citeN{truszczynski11} for epistemic logic programs satisfy epistemic splitting.\qed
\end{corollary}

\subsection{Epistemic splitting in the semantics by~\protect\citeN{gelfond11a}}
\label{sec:espliting.g11}

As mentioned in the introduction, \citeN{gelfond11a} revisited the semantics of epistemic logic programs in an attempt to get rid of unintended world views.
This approach, which we will denote here as G11, consisted in a modification of the reduct so that positive subjective literals were not completely removed:

\begin{definition}[G11-world views]
Given a logic program~$\Pi$, its G11-reduct with respect to a non-empty set of interpretations~$\wv$ 
is a program obtained by:
\begin{enumerate}
\item replacing by $\bot$ every expression $\K l$ such that $\wv \not\models \K l$,
\item removing all other occurrences of subjective literals of the form $\neg\K l$, and
\item replacing all other occurrences of subjective literals of the form $\K l$ by $l$.
\end{enumerate}
A non-empty set of interpretations~$\wv$ is a G11-world view of $\Pi$ iff $\wv$ is the set of all stable models of the G11-reduct of $\Pi$ with respect to~$\wv$.\qed
\end{definition}

The following counterexample shows that this semantics does not satisfy epistemic splitting:

\begin{counterexample}\label{ex:g11.no-splitting}
Let $\newprogram\label{prg:g11a}$ be the program containing the following two rules:
\begin{gather*}
a \vee b 
\hspace{2cm}
c \leftarrow  \K a   
\end{gather*}
Then, all semantics mentioned in this paper agree that this program has a unique world view $\wv_{\ref{prg:g11a}} = \cset{ \set{a}, \set{b} }$.
In the particular case of the G11 semantics,
we can check this fact by applying the G11-reduct. Observe that the resulting program only contains the disjunction $a \vee b$ and, thus, $\wv_{\ref{prg:g11a}}$ is a world view of program~\program\ref{prg:g11a}.
For semantics that satisfy epistemic splitting, like G91, we can see that $U_{\ref{prg:g11a}} = \set{a,b}$ is a splitting set of~\program\ref{prg:g11a}
and, thus, we can compute the unique world view $\wv_{\ref{prg:g11a}}$ of the bottom part, $a \vee b$, and use it to simply the top part obtaining this same world view.
To show that G11 does not statify epistemic splitting, let $\newprogram\label{prg:g11b}$ be the program
\begin{gather*}
a \vee b 
\hspace{2cm}
c \leftarrow  \K a  
\hspace{2cm}
\leftarrow \Not c
\end{gather*}
obtained by  adding the last constraint
to our program~\program\ref{prg:g11a}.
Note that $U_{\ref{prg:g11a}}$ is also a splitting set of this program dividing the program
into a bottom part containing the disjunction $a \vee b$ and a top part containing the other two rules.
Then, the result of simplifying the top part with respect to the bottom world view~$E_U(\program\ref{prg:g11b},\wv_{\ref{prg:g11a}})$ is
\begin{gather*}
c \leftarrow  \bot
\hspace{2cm}
\leftarrow \Not c
\end{gather*}
which has has no stable model and, thus, no world view.
As a result, we cannot generate any world view by combining with $\wv_{\ref{prg:g11a}}$ and we obtain that $\program\ref{prg:g11b}$ has no world views at all.
Since this happens for any semantics satisfying epistemic splitting and supra-ASP, it means that there is no \mbox{G91-world} view either.
On the other hand, we can check that this program has the G11-world view~$\wv_{\ref{prg:g11b}} = \cset{ \set{a,c} }$.
To see this fact, note that the G11 reduct of~$\program\ref{prg:g11b}$ with respect to~$\wv_{\ref{prg:g11b}}$
is
\begin{gather*}
a \vee b 
\hspace{2cm}
c \leftarrow  a
\hspace{2cm}
\leftarrow \Not c
\end{gather*}
which has a unique stable model~$\set{a,c}$.
To finish our example note that~$\wv_{\ref{prg:g11b}}$
was not a world view of the first program~$\program\ref{prg:g11a}$ since its G11-reduct is the program
\begin{gather*}
a \vee b 
\hspace{2cm}
c \leftarrow  a
\end{gather*}
which has stable models~$\set{a,c}$ and $\set{b}$.
This shows that, in the G11 semantics, adding a constraint that depends on a subjective literal may lead to justify that subjective literal.\qed
\end{counterexample}

\subsection{Epistemic splitting in the semantics by~\protect\citeN{kawabagezh15} and~\protect\citeN{sheeit17a}}
\label{sec:espliting.k15}

A new revision of the semantics of epistemic logic programs was introduced by~\mbox{\citeN{kawabagezh15}} (denoted here as K15) who proposed a new variation of the reduct as follows\footnote{We use here the alternative modal reduct introduced in the Appendix~C of \cite{kawabagezh15} because of its simplicity.}:

\begin{definition}[K15-world views]
Given a logic program~$\Pi$, its K15-reduct with respect to a non-empty set of interpretations~$\wv$ 
is a program obtained by:
\begin{enumerate}
\item replacing by $\bot$ every expression $\K l$ such that $\wv \not\models \K l$,
\item replacing all other occurrences of the expression $\K l$ by $l$.
\end{enumerate}
A non-empty set of interpretations~$\wv$ is a K15-world view of $\Pi$ iff $\wv$ is the set of all stable models of the K15-reduct of $\Pi$ with respect to~$\wv$.\qed
\end{definition}

To show that this semantics does not satisfy epistemic splitting we can use the same program of Counterexample~\ref{ex:g11.no-splitting}
because the K15-reduct of~$\program\ref{prg:g11b}$ with respect to~$\wv_{\ref{prg:g11b}} = \cset{ \set{a,c} }$
is the same as its G11-reduct and, thus, we obtain the same world view violating epistemic splitting.
Note that for programs where all subjective literals are positive, the G11 and K15-reducts always coincide.

A further revision, denoted S17, was later introduced by~\citeN{sheeit17a}.
Although the original formulation of S17 was done in terms of a so-called \emph{epistemic negation} operator, $\mathbf{not}\ l$, we use here an alternative characterisation obtained by \citeN{solekale17a}.
This characterisation amounts to translating $\mathbf{not}\ l$ as $\Not \K l$ and selecting a class of minimal K15-world views.  

\begin{definition}[S17-world views]
Let~$\Pi$ be a logic program~$\Pi$ and $E_\Pi$ be the set of epistemic literals that contains $\Not \K l$ for every epistemic literal of the form $\K l$ that occurs in~$\Pi$.
Let $\Phi_\wv \eqdef \setm{ L \in E_\Pi }{\wv \models L}$ be the subset of $E_\Pi$ satisfied by non-empty set of interpretations~$\wv$.
Then, a non-empty set of interpretations~$\wv$ is a S17-world view iff is a K15-world view and there is no other K15-world view $\wv'$ such that $\Phi_{\wv'} \supset \Phi_{\wv}$.\qed
\end{definition}
Notice that S17 coincides with K15 when the latter yields a unique world view, since the minimisation in that case has no effect.
Therefore, Counterexample~\ref{ex:g11.no-splitting} also applies to S17 because it has a unique K15-world view.

It is also worth to mention that~\citeN{leckah18b} have observed that the K15 and S17 semantics do not satisfy subjective constraint monotonicity and, given Theorem~\ref{thm:splitting->constraint.monotonicity}, this can also be used as another counterexample for epistemic splitting.

\begin{counterexample}[From~\protected\citeNP{leckah18b}]\label{ex:k15.no-splitting}
Let $\newprogram\label{prg:k18}$ be the program containing the following two rules:
\begin{gather*}
a \vee b 
\hspace{2cm}
\leftarrow  \Not \K a
\end{gather*}
which has a unique K15 and S17-world view $\cset{ \set{a } }$.
Note however that the program containing just $a \vee b$ is objective and has two stable models~$\set{a}$ and~$\set{b}$.
Thus, it has the unique world view~$\cset{ \set{a }, \set{b} }$,
which does not satisfy the subjective constraint~$\leftarrow \Not  \K a$.
\qed
\end{counterexample}

In other words, in the K15 and S17 semantics adding constraints intended to remove some world views may lead to new world views.
Note that there may be semantics where epistemic splitting fails but, still, subjective constraint monotonicity is satisfied.
In fact, the latter was proved for G11 in~\cite[Proposition~8]{fandinno19}.

\subsection{Epistemic splitting in the semantics by~\protect\citeN{faheir15a}}
\label{sec:l15}

\citeN{faheir15a} introduced a different semantics for epistemic logic programs, denoted here as F15, which is based on a combination of Equilibrium Logic~\cite{pearce96a} with the modal logic~S5.
They semantics covers arbitrary S5-formulas, but we restrict our presentation here to the syntax of epistemic logic programs.

\begin{definition}
An EHT-interpretation is a pair $\tuple{\wv,h}$ where $\wv$ is a non-empty set of interpretations and $h: \wv \longrightarrow 2^\at$ is a function mapping each interpretation~$T$ to some subset of atoms such that $h(T) \subseteq T$.\qed
\end{definition}

Intuitively,
the set $\wv$ contains the valuation in the ``there'' world while the
set~$\setm{h(I)}{ I \in \wv }$ corresponds to the valuation in the ``here'' world.

Satisfaction of formulas with respect to EHT-interpretations is defined in a similar way as with respect to modal interpretations (see Section~\ref{sec:back}).
An EHT-interpretation $\tuple{\wv,h}$ \emph{satisfies} a literal at point $I \in \wv$,
if one of the following conditions hold
\begin{enumerate}
\item $\tuple{\wv,h,I} \models \top$,
\item $\tuple{\wv,h,I} \not\models \bot$,
\item $\tuple{\wv,h,I} \models a$ if $a \in h(I)$, for any atom $a \in \At$,
\item $\tuple{\wv,h,I} \models \K l$ if $\tuple{\wv,h,I'} \models l$ for all $I' \in \wv$,
\item $\tuple{\wv,h,I} \models \M l$ if $\tuple{\wv,h,I'} \models l$ for some $I' \in \wv$, and
\item $\tuple{\wv,h,I} \models \Not L$ if $\tuple{\wv,id,I} \not\models L$.
\end{enumerate}
where $id : \wv \longrightarrow 2^\at$ is the identity function mapping, that is, $id(T) = T$ for every $T \in \wv$.
Then, $\tuple{\wv,h,I}$ \emph{satisfies} a rule~$r$ at point $I \in \wv$, also written $\tuple{\wv,h,I} \models r$, 
iff $\tuple{\wv,h,I} \models L$ for every literal $L \in \Body(r)$
implies $\tuple{\wv,h,I} \models a$ for some atom $a \in \Head(r)$.
We say that $\tuple{\wv,h}$ is an \emph{EHT-model} of a program~$\Pi$, written $\tuple{\wv,h}\models \Pi$, if it satisfies all its rules at all points $I \in \wv$.

An EHT-interpretation $\tuple{\wv,h}$ is said to be \emph{total on} a set $\wx \subseteq \wv$ iff $h(I) = I$
for every $I \in \wx$.
It is said to be \emph{total} iff it is total on $\wv$.
Then, equilibrium models are defined as follows:

\begin{definition}
A total EHT-model $\tuple{\wv,id}\models \Pi$ of some program~$\Pi$ is an
\emph{equilibrium EHT-model} iff 
there is no other EHT-model $\tuple{\wv,h}\models \Pi$ such that $h(I) \subset I$ for some $I \in \wv$.\qed
\end{definition}

Then, the F15-world views are defined as a selection of equilibrium EHT-models.
For that, we need to introduce the following definitions:

\begin{definition}
Given a logic program~$\Pi$, a non-empty set of interpretations and a subset $\wx \subseteq \wv$,
we write $\wv,\wx \models^* \Pi$ iff the following two conditions are satisfied:
\begin{enumerate}
\item $\tuple{\wv,id,I} \models \Pi$ for all $I \in \wx$, and
\item if $\tuple{\wv,h,I} \models \Pi$ for some $I \in \wv$ such that $\tuple{\wv,h}$ is total on $\wv\setminus\wx$,
	then $\tuple{\wv,h}$ is total.\qed
\end{enumerate}
Then, for any two pairs of non-empty set of interpretations $\wv$ and $\wv'$
we write
$\wv \leq_\Pi \wv'$ iff
\begin{gather*}
\wv \cup \set{ I}, \wv \models^* \Pi \quad\text{implies}\quad \wv' \cup \set{I},\wv' \models^* \Pi
\end{gather*}
for every $I$ such that $I$ belongs to some equilibrium EHT-model of~$\Pi$.\qed
\end{definition}

\begin{definition}
Given a logic program~$\Pi$, equilibrium EHT-model~$\wv$ is called an \emph{F15-world view} iff there is no other equilibrium EHT-model~$\wv'$ such that $\wv \subset \wv'$ or\footnote{As usual $\wv <_\Pi \wv'$ stands for $\wv \leq_\Pi \wv'$
and $\wv \not\leq_\Pi \wv'$.} $\wv <_\Pi \wv'$.\qed
\end{definition}

The following counterexample shows that this semantics does not satisfy subjective constraint monotonicity and, thus, it does not satisfy epistemic splitting either:

\begin{counterexample}\label{ex:l15.no-splitting}
Let us recall the program~$\program\ref{prg:g11a}$ from Counterexample~\ref{ex:k15.no-splitting}:
\begin{gather*}
a \vee b 
\hspace{2cm}
\leftarrow  \Not \K a
\end{gather*}
which has a unique equilibrium EHT-model~$\cset{ \set{a } }$, which therefore is trivially its unique~\mbox{F15-world} view.
Then, to show that the F15 semantics does not satisfy subjective constraint monotonicity we just need to show that 
$\cset{ \set{a } }$ is not an~\mbox{F15-world} view of the program consisting only of $a \vee b$.
Note that $a \vee b$ has three equilibrium EHT-models: $\cset{ \set{a } }$, $\cset{ \set{b } }$
and
$\cset{ \set{a }, \set{b} }$.
It is easy to see that
$\cset{ \set{a }, \set{b} }$
is the unique $\subset$-maximal one and, thus, the unique \mbox{F15-world} view.\qed
\end{counterexample}

In other words, as happened in the K15 and S17 semantics,
adding constraints intended to remove some world views may also lead to new world views in the F15 semantics.

\subsection{Epistemic splitting in the semantics by~\protect\citeN{cafafa19b}}
\label{sec:espliting.faeel}

The semantics of \emph{Founded Autoepistemic Equilibrium Logic} (we denote here as C19) was introduced in~\cite{cafafa19b} as a 
combination of Equilibrium Logic with Moore's Autoepistemic Logic.
As happened with F15~\cite{faheir15a}, C19 is also applicable to the full syntax of modal logic S5.
However, when we restrict ourselves to the syntax of epistemic logic programs, C19 can be alternatively defined in terms of G91 by imposing an additional condition called \emph{foundedness}.
For simplicity, we focus here in this alternative definition.

Let us start by defining $\Bodyrp(r)$ and $\Bodymp(r)$ as the set of all positive objective literals in the body and
the set of all atoms occurring in positive subjective literals, respectively. 
Then, we can define an \emph{unfounded set} as follows:

\begin{definition}[Unfounded set]\label{def:unfoundedset}
Let $\Pi$ be a program and $\wv$ a belief view.
An \emph{unfounded set} $\us$ with respect to $\Pi$ and $\wv$ is a non-empty set of pairs where, for each $\tuple{X,I} \in \us$, we have that $X$ and $I$ are sets of atoms and there is no rule $r \in \Pi$ with $\Head(r) \cap X \neq \emptyset$ satisfying:
\begin{enumerate}[itemsep=3pt,topsep=2pt]
\item $\kdint{I}{\wv} \models \Body(r)$
	\label{item:1:def:unfounded}
\item $\Bodyrp(r) \cap X = \emptyset$
	\label{item:2:def:unfounded}
\item $(\Head(r) \setminus X) \cap I = \emptyset$
	\label{item:3:def:unfounded}
\item $\Bodymp(r) \cap Y = \emptyset$ with $Y = \bigcup \setm{ X' }{ \tuple{X',I'} \in \us }$.\qed
	\label{item:4:def:unfounded}
\end{enumerate} 
\end{definition}
The definition works in a similar way to standard unfounded sets~\cite[Definition~3.1]{lerusc97a}.
In fact, the latter corresponds to the first three conditions above, except that we use $\kdint{I}{\wv}$ to check $\Body(r)$, as it may contain now subjective literals.
Intuitively, each $I$ represents some potential belief set (or stable model) and $X$ is some  set of atoms without a ``justifying'' rule, that is, there is no $r \in \Pi$ allowing a positive derivation of atoms in $X$.
A rule like that should have a true $\Body(r)$ (condition~\ref{item:1:def:unfounded}) but not because of positive literals in $X$ (condition~\ref{item:2:def:unfounded}) and is not used to derive other head atoms outside $X$ (condition~\ref{item:3:def:unfounded}).
The novelty in the definition \mbox{by~\citeN{cafafa19b}} is the addition of condition~\ref{item:4:def:unfounded}: to consider $r$ a justifying rule, we additionally require not using any positive literal $\bL a$ in the body such that atom $a$ also belongs to any of the unfounded components $X'$ in $\us$.

\begin{definition}[Founded world view]\label{def:unfounded}
Let $\Pi$ be a program and $\wv$ be a belief view.
We say that $\wv$ is \emph{unfounded} if there is some 
unfounded-set~$\us$ s.t.,
for every $\tuple{X,I} \in \us$,
we have $I \in \wv$ and $X \cap I \neq \emptyset$.
$\wv$ is called \emph{founded} otherwise.\qed
\end{definition}

Now, as said before, we can define C19 in terms of G91 using the following result from~\cite{cafafa19b}: 

\begin{theorem}[Main Theorem in~\citeNP{cafafa19b}]
Given any epistemic logic program~$\Pi$, its C19-world views are its founded G91-world views.\qed
\end{theorem}

Finally, this relation to G91 was recently used in~\cite{fandinno19} to prove that:

\begin{theorem}[Main Theorem in~\citeNP{fandinno19}]
The C19 semantics for epistemic logic programs satisfies the epistemic splitting property.\qed
\end{theorem}

\section{Applying Epistemic Splitting to Conformant Planning}
\label{sec:conf}

The idea of encoding the problem of finding a conformant plan as the task of obtaining a world view was first introduced by~\cite{kawabagezh15}.
To conclude the exploration of consequences of epistemic splitting, let us consider its possible application to simplify the representation of conformant planning problems.
To this aim, consider the following simple example.
\begin{example}
To turn on the light in a room, we can toggle one of two lamps $l_1$ or $l_2$.
In the initial state, lamp $l_1$ is plugged but we ignore the state of $l_2$.
Our goal is finding a plan that guarantees we get light in the room in one step.
\end{example}
A logic program that encodes this scenario for a single transition\footnote{For simplicity, we omit time arguments or inertia, as they are not essential for the discussion.} could be $\newprogram\label{prg:4}$:
\begin{gather*}
\begin{IEEEeqnarraybox}{c}
plugged(l_1) \\
plugged(l_2) \vee \sneg plugged(l_2) 
\end{IEEEeqnarraybox}
\hspace{30pt}
\begin{IEEEeqnarraybox}{rCl}
light & \leftarrow & toggle(L), plugged(L) \\
\bot & \leftarrow & toggle(l_1), toggle(l_2)
\end{IEEEeqnarraybox}
\end{gather*}
for $L \in \{l_1,l_2\}$. 
As we can see, $toggle(l_1)$ would constitute a conformant plan, since we obtain $light$ regardless of the initial state, while this does not happen with plan $toggle(l_2)$.
In order to check whether a given sequence of actions $A_0,\dots, A_n$ is a valid conformant plan one would expect that, if we added those facts to the program, a subjective constraint should be sufficient to check that the goal holds in all the possible outcomes.
In our example, we would just use:
\begin{eqnarray}
\bot \leftarrow \Not \K light \label{f:subcons}
\end{eqnarray}
and check that the program $\program\ref{prg:4} \cup \{toggle(L)\} \cup \{\eqref{f:subcons}\}$ has some world view, varying $L \in \{l_1,l_2\}$.
Subjective constraint monotonicity guarantees that the addition of this ``straighforward'' formalisation has the expected meaning.

This method would only allow testing if the sequence of actions constitutes a conformant plan, but does not allow generating those actions.
A desirable feature would be the possibility of applying the well-known ASP methodology of separating the program into three sections: generate, define and test.
In our case, the ``define'' and the ``test'' sections would respectively be $\program\ref{prg:4}$ and \eqref{f:subcons}, but we still miss a ``generate'' part, capable of considering different alternative conformant plans.
The problem in this case is that we cannot use a simple choice:
\begin{eqnarray*}
toggle(L) \vee \sneg toggle(L)
\end{eqnarray*}
because this would allow a same action to be executed in some of the stable models and not executed in others, all inside a \emph{same} world view.
Let us assume that our epistemic semantics has some way to non-deterministically generate a world view in which either $\K a$ or $\K \Not a$ holds using a given set of rules $Choice(a)$.
For instance, in the G91-semantics, we could just have the rule $Choice(a) = \set{ a \leftarrow \Not \K \Not a}$, though other semantics may have alternative ways of expressing this intended behaviour.
Then, take the program $\newprogram\label{prg:5}$ consisting of rules 
\begin{eqnarray}
Choice(toggle(L)) \label{ff:choices}
\end{eqnarray}
with $L \in \set{l_1,l_2}$
plus $\program\ref{prg:4}$ and \eqref{f:subcons}.
If our semantics satisfies epistemic splitting, it is safe to obtain the world views in three steps: generate first the alternative world views for $toggle(l_1)$ and $toggle(l_2)$ using \eqref{ff:choices}, apply $\program\ref{prg:4}$ and rule out those world views not satisfying the goal $light$ in all situations using \eqref{f:subcons}. 
To fulfill the preconditions for applying splitting, we would actually need to replace objective literal $toggle(L)$ by $\K toggle(L)$ in all the bodies of $\program\ref{prg:4}$, but this is safe in the current context.
Now, we take the bottom program \label{f:choices} to obtain 4 possible world views $\wv_0=[\{toggle(l_1)\}]$, $\wv_1=[\{toggle(l_2)\}]$, $\wv_2=[\{toggle(l_1),toggle(l_2)\}]$ and $\wv_3=[\emptyset]$.
When we combine them with the top $\program\ref{prg:4}$ we obtain $W'_0$ consisting of two stable models:
\begin{eqnarray*}
\hspace{0.25cm}\{toggle(l_1),plugged(l_2),light,\dots\}
\hspace{1cm}
\{toggle(l_1),\sneg plugged(l_2),light,\dots\}
\end{eqnarray*}
and $W'_1$ consisting of other two stable models:
\begin{eqnarray*}
\hspace{0.15cm}
\{toggle(l_2),plugged(l_2),light,\dots\}
\hspace{1.25cm}
\{toggle(l_2),\sneg plugged(l_2),\dots\}
\hspace{0.5cm}
\end{eqnarray*}
where the latter does not contain $light$.
Finally, constraint \eqref{f:subcons} would rule~out~$\wv'_1$.

In our opinion, this representation is more natural than the representation given in~\cite{kawabagezh15}, whose semantics does not satisfy neither splitting nor subjective constraint monotonicity.
For the sake of comparison, in that approach, the ``generate'' part would consist of the following three rules:
\begin{gather*}
toggle(l_1) \leftarrow \Not\K\Not toggle(l_1)
\hspace{50pt}
toggle(l_2) \leftarrow \Not\K\Not toggle(l_2)
\\[5pt]
\leftarrow toggle(l_1), \, toggle(l_2)
\end{gather*}
while, for the test part, the subjective constraint~\eqref{f:subcons},
would need to be replaced by the following pair of rules:
\begin{gather*}
\leftarrow\mathit{light},\, \Not \K \mathit{light}
\hspace{1cm}
\leftarrow \K\Not \mathit{light}
\end{gather*}
In the K15 semantics it is not clear how to modularly define a set of rules $Choice(a)$ that produce the intended effect nor how to generalise it to allow concurrent actions.
Note that, if we consider a generating part without the constraint, that is,
\begin{gather*}
toggle(l_1) \leftarrow \Not\K\Not toggle(l_1)
\hspace{0.751cm}
toggle(l_2) \leftarrow \Not\K\Not toggle(l_2)
\end{gather*}
we would obtain a unique world view $\cset{ \set{toggle(l_1),\, toggle(l_2)} }$, missing thus the world view 
$\cset{ \set{toggle(l_1) } }$ which corresponds to the simplest conformant plan.





\black

\section{Conclusions}

We have introduced a formal property for semantics of epistemic specifications.
This property, we called \emph{epistemic splitting}, has a strong resemblance to the splitting theorem well-known for regular ASP programs.
Epistemic splitting can be applied when we can divide an epistemic logic program into a bottom part for a subset $U$ of atoms and a top part, that only refers to atoms in $U$ through subjective literals (those using modal epistemic operators).
When this happens, epistemic splitting allows obtaining the world views of the program in two steps: first, computing the world views of the bottom and, second, using each bottom world view $\wv$ to replace subjective literals for atoms in $U$ in the top by their truth value with respect to~$\wv$.

We have studied several consequences of epistemic splitting. For instance, if the program is stratified with respect to subjective literals then it will have a unique world view, at most.
Another consequence is that constraints only consisting of subjective literals will have a monotonic behaviour, ruling out world views that satisfy the constraint body.

Our study of the main semantics in the literature has shown that only the original semantics  (G91; \citeNP{gelfond91a}), its generalisations to propositional formulas~\cite{wanzha05a,truszczynski11} and Founded Autoepistemic Equilibrium Logic (C19;~\citeNP{cafafa19b}), satisfy epistemic splitting while the rest of approaches we considered do not, as we showed with counterexamples.
\begin{table}
\centering
\begin{tabular*}{11.25cm}{  @{}p{6.5pc} @{\hskip1cm}  @{}p{3pc} @{}p{3pc} @{}p{3pc} @{}p{3pc} @{}p{3pc} @{}p{3pc} @{}p{3pc} @{}p{3pc} @{}p{3pc} }
&G91 & G11 
& F15 
& K15 & S17
& C19
\\\hline
Supra-S5& 
\checkmark & \checkmark 
&\checkmark
& \checkmark & \checkmark 
& \checkmark
\\\hline
Supra-ASP& 
\checkmark & \checkmark 
& \checkmark
& \checkmark & \checkmark
& \checkmark 
\\\hline
Subjective constraint monotonicity&
\checkmark & \checkmark 
& 
& &
& \checkmark
\\\hline
Splitting
&\checkmark &  
& 
& &
& \checkmark
\\\hline
Foundness&
& & & &
& \checkmark
\\\hline
\end{tabular*}
	\caption{Summary of properties in different semantics.}
	\label{table:summary}
\end{table}
Table~\ref{table:summary} is taken from~\cite{fandinno19} and summarises the known results for different semantics with respect to the properties discussed through this paper plus the \emph{foundness} property from~\cite{cafafa19b}.
Recall that most of the semantics were born motivated by the existence of self-supported world views in the G91 semantics:
for instance, the program consisting of the single rule $a \leftarrow \K a$ yields two world views $[\emptyset]$ and $[\{a\}]$ but the latter justifies the atom $a$ by the mere assumption of $\K a$ without further evidence, something that seems counterintuitive.
The foundness property precisely characterises this problem in terms of unfounded sets and, as shown in the table, is  currently satisfied only by C19, which in addition keeps the good behaviour of G91 with respect to epistemic splitting.

We have also explored how epistemic splitting may facilitate the simple application of the generate-define-test methodology, well-known in ASP, to the formalisation of conformant planning.
Finally, as said in the introduction, a different kind of epistemic splitting had also been proved for G91 in~\cite{watson00}, reinforcing the idea that this semantics can be interpreted in a modular way.
Notice that the sets of programs that can be split under these two definitions are incomparable.

Recall that our main motivation when defining the splitting property was to establish a minimal requirement that seemed reasonable.
Still, this property was not satisfied by most semantics.
On the other hand, the semantics that do satisfy this property may satisfy even stronger notions of splitting, which will be matter of future study.

\paragraph{Acknowledgements.}
We are thankful to Michael Gelfond, Richard Watson, Patrick T. Kahl,
and the anonymous reviewers of the \emph{Seventeenth International Workshop on Non-monotonic Reasoning} (NMR'18) and the \emph{Fifteenth International Conference on Logic Programming  and
Nonmonotonic Reasoning} (LPNMR'19) where preliminary versions of this work were presented, for their comments and suggestions that have helped improving the paper substantially.
We are also thankful to the organisers  of LPNMR'19 for their support in preparing this extended version.

\bibliographystyle{acmtrans}
\bibliography{krr,refs,procs}

\end{document}